\documentclass[runningheads]{llncs}
\usepackage[ruled,vlined]{algorithm2e}
\usepackage{todonotes}
\usepackage{amsmath}
\usepackage{float}
\usepackage{makecell}
\usepackage{paralist}
\usepackage{comment}

\usepackage[T1]{fontenc}
\usepackage{graphicx}
\usepackage{hyperref}
\usepackage{paralist}
\usepackage{color}
\usepackage{amssymb}
\usepackage{subcaption}

\newcommand{\mypar}[1]{\smallskip\noindent\textbf{#1.}}

\begin{document}
\title{Discovery of Decision Synchronization Patterns from Event Logs}
%
%
\author{Tijmen Kuijpers\inst{1} \and
Karolin Winter\inst{1} \and
Remco Dijkman\inst{1}}
\authorrunning{Tijmen Kuijpers et al.}
%
\institute{Eindhoven University of Technology, Department of Industrial Engineering and Innovation Sciences, Eindhoven, The Netherlands 
\email{\{t.p.kuijpers|k.m.winter|r.m.dijkman\}@tue.nl}}

\maketitle              
\begin{abstract}

Synchronizing decisions between running cases in business processes facilitates fair and efficient use of resources, helps prioritize the most valuable cases, and prevents unnecessary waiting. Consequently, decision synchronization patterns are regularly built into processes, in the form of mechanisms that temporarily delay one case to favor another. These decision mechanisms therefore consider properties of multiple cases at once, rather than just the properties of a single case; an aspect that is rarely addressed by current process discovery techniques. To address this gap, this paper proposes an approach for discovering decision synchronization patterns inspired by supply chain processes. These decision synchronization patterns take the form of specific process constructs combined with a constraint that determines which particular case to execute. We describe, formalize and demonstrate how the constraint for four such patterns can be discovered. We evaluate our approach in two artificial scenarios. First, with four separate process models each containing a single decision synchronization pattern, i.e., we demonstrate that our approach can discover every type of pattern when only this one type is present. Second, we consider a process model containing all four decision synchronization patterns to show generalizability of the approach to more complex problems. For both scenarios, we could reliably retrieve the expected patterns.
\end{abstract}
%
%
%


\section{Introduction}
Synchronizing decisions across multiple running cases is crucial in business processes to ensure efficient use of resources or help prioritize the most valuable cases. For example, in supply processes, one case often needs to wait for another~\cite{Garca-Alcaraz2021} to ensure that a case does not incur a penalty for late delivery or perishable ingredients go bad. As a simple example, consider the scenario from Fig. \ref{fig: Priority process model}. In this scenario, jobs arrive every 5 time units. Upon arrival, they need to be pre-processed, which also takes 5 time units. Jobs have a value associated with them, which is uniformly distributed between 100 and 1000. After pre-processing, jobs must be handled, which takes 7 time units. There is one resource available to handle jobs. Given the lead times in this process, not all jobs can be handled, and prioritizing higher-value jobs will result in a higher total value. While the model does not include a constraint to handle high-value jobs first, it becomes clear from the log that such a constraint is enforced. For example, after pre-processing case 1 (i.e., at time 5), case 2 starts pre-processing, but case 1 is not immediately handled. The reason is that the high-value case 3 is already in the arrival place at this stage, and the handling transition is constrained until this case is pre-processed, so case 3 can be handled first\footnote{This process is modeled in a dialect of Colored Petri Nets (CPN) that uses Python for the specification of variables and expressions~\cite{Dijkman2024}.}. These constraints cannot be discovered by existing process mining algorithms. Therefore, the goal of this paper is to present a technique to discover them from event data.

\begin{minipage}{0.48\textwidth}
  \centering
  \includegraphics[width=\linewidth]{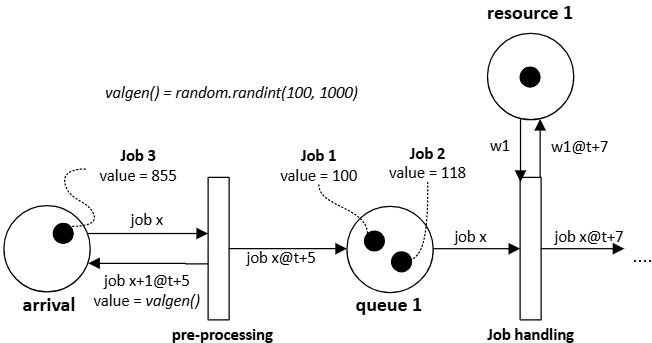}
  \captionof{figure}{Priority Process Model}\label{fig: Priority process model}
\end{minipage}
\hfill
\begin{minipage}{0.48\textwidth}
  \centering
  \scriptsize
\begin{tabular}{lllll}
\hline
\textbf{case} & \textbf{activity}        & \textbf{start} & \textbf{complete} & \textbf{value} \\
\hline
1 & pre-processing & 0  & 5  & 100 \\
2 & pre-processing & 5  & 10 & 118 \\
3 & pre-processing & 10 & 15 & 855 \\
4 & pre-processing & 15 & 20 & 801 \\
3 & handling      & 15 & 22 & 855 \\
5 & pre-processing & 20 & 25 & 146 \\
4 & handling      & 22 & 29 & 801 \\
6 & pre-processing & 25 & 30 & 222 \\
7 & pre-processing & 30 & 35 & 317 \\
\hline
\end{tabular}
\captionof{table}{Priority Process Event Log}\label{tab: Priority process event log}
\end{minipage}

To achieve this goal, decision synchronization patterns in supply chain processes are identified and modeled. A discovery technique is developed that can extract decision-synchronization patterns from event logs. We demonstrate that the technique works by applying it to event logs to which decision constraints are added, showing that these same constraints can be discovered again from the event log. Against this background, the contributions of this paper are: (1) Formalization of decision synchronization patterns using Petri nets. (2) Development of a discovery algorithm, enabling the extraction of these decision synchronization patterns from event logs.

The remainder of this paper is organized as follows. Sect. \ref{sec:background} introduces the theoretical background of this paper. In Sect. \ref{sec:dec_sync_pattern} we formally describe our approach and motivate the need for decision synchronization patterns. The decision synchronization patterns are described and formalized in Sect. \ref{sec:Decision Synchronization as Constraints}. We also describe the data required and the algorithm used to discover the patterns in this section. The pattern discovery algorithm is evaluated for single-pattern and multi-pattern process models in Sect. \ref{sec:evaluation}. The paper concludes with related work in Sect. \ref{sec:rel_work} and a summary of findings and outline of future work in Sect. \ref{sec:conclusion}.

\section{Background}\label{sec:background}

For simplicity and understanding we assume that a structurally correct process model is given as a simplified colored Petri net (CPN). The full CPN theory can be used as well. Let $\mathcal{C}$ be the set of all colorsets (i.e., datatypes), $\mathcal{P}$ the set of all places, and $\mathcal{T}$ the set of all transitions. The sets are mutually exclusive.


\begin{definition}[Simplified Colored Petri Net, Marking]\label{def:cpn}
A simplified colored Petri net is a tuple $(P, T, F, C)$, where:
$P \subseteq \mathcal{P}$ are the places of the net, $T \subseteq \mathcal{T}$ its transitions, $F \subseteq (P\times T) \cup(T\times P)$ its flows, and $C: P \rightarrow \mathcal{C}$ the function that connects each place to its colorset. The current state of the Petri net is represented as a marking $M$ mapping each place $p\in P$ to a multiset $C(p) \rightarrow \mathbb{N}$.
\end{definition}

For a transition $t\in T$, the set of incoming places, denoted $\bullet t$, is $\{p\in P|(p,t)\in F\}$. The set of outgoing places, $t\bullet$, is defined analogously. E.g., the net in Fig. \ref{fig: Priority process model} has places $\{\textrm{arrival}, \textrm{q1}, \textrm{r1}\}$, colorsets $\textrm{jobs} = \{\textrm{job1}, \textrm{job2}, ..., \textrm{jobN}\}$, $\textrm{workers} = \{\textrm{w1}\}$, and $\textrm{value} = [100,\ldots,1000]$. The marking in the figure is $\textrm{arrival}\mapsto(\textrm{job 3},855)^1$, $\textrm{q1}\mapsto(\textrm{job 1},100)^1(\textrm{job 2},118)^1$ and $\textrm{r1}\mapsto(w1)^1$.

\begin{definition}[Binding, Enabled Transition]\label{def:binding}
For a Petri net $(P, T, F, C)$ and a marking $M$, a binding of a transition $t\in T$ under marking $M$ is a function $Y$ that maps each incoming place $p\in \bullet t$ to a color from $C(p)$ that is currently in the marking, i.e., $\forall p \in \bullet t: Y(p) \dot\in M(p)$. The transition is enabled, denoted $M\stackrel{Y,t}{\rightarrow}$, if such a binding exists. An additional guard condition $G$ may be imposed, such that the transition is only enabled, denoted $M\stackrel{Y,t,G}{\rightarrow}$, if the condition, with its variables valuated according to the binding (i.e., $G\langle Y\rangle$), also holds.
\end{definition}

For example, consider the transition $\textrm{job\_handling}$ in Fig. \ref{fig: Priority process model}. Assume that the transition is only enabled when a job has a value higher than 500. Then the binding $\{\textrm{r1}\mapsto \textrm{w1}$, and $\textrm{q1}\mapsto (\textrm{job 1},855)\}$ enables the transition.

We aim to learn constraints from an event log. Let $\mathcal{E}$ be the set of possible events, $\mathcal{I}$ the set of identifiers, $\mathcal{R}$ the set of resources, $\mathcal{V}$ the set of data values, and let $\mathcal{D}$ be the set of datatypes, such that $\forall D\in \mathcal{D}: D\subseteq \mathcal{V}$.

\begin{definition}[Event, Log]\label{def:eventlog}
An event $e\in\mathcal{E}$ is something that happens: $e$, $id(e) \in I$ represents the case to which the event belongs, $l(e)$ its label, $ts(e)$ its start time, $tc(e)$ its completion time, $r(e) \in \mathcal{R}$ the resource that executed it, and $d(e)$ the data associated with the event. An event log $L \subseteq \mathcal{E}*$, enriched with event data $\{D_1, D_2, \ldots, D_n\} \subseteq \mathcal{D}$, is a set of sequences of events, also called traces, such that for each trace $\langle e_1, e_2, \ldots, e_n \rangle$, and $0 \leq i < j \leq n$, it holds that $id(e_i)=id(e_j)$ and $tc(e_i)\leq tc(e_j)$ and for each event $e$ it holds that $d(e) \in D_1 \times D_2 \times \ldots \times D_n$.
\end{definition}

We do not make any assumptions on how the Petri net is found. For example, any process mining algorithm can be used for those purposes. However, we do assume that the Petri net corresponds to the log in the following manner.


\begin{definition}[Case Place, Resource Place, Task Transition]\label{def:task_representation}
For a Petri net $(P,T,F,C)$, some places can be chosen to represent cases, some places to represent resources. We identify these places as $P^C \subseteq P$ and $P^R \subseteq P$, respectively. Similarly, some transitions may be chosen to represent tasks, denoted $T^T \in T$. For each resource place $p^r \in P^R$ the colorset must be a set of resources, i.e., $C(p^r) \subseteq \mathcal{R}$. Let $\{D_1, D_2, \ldots, D_n\} \subseteq \mathcal{D}$ be a collection of datatypes. For each case place $p^c \in P^C$, the colorset must be a set of tuples consisting of a case identifier and data values, i.e., $C(p^r) \subseteq \mathcal{I} \times D_1 \times D_2 \times \ldots \times D_n$. We also say that the Petri net is enriched with case data $\{D_1, D_2, \ldots, D_n\} \subseteq \mathcal{D}$.
\end{definition}

To mine decision synchronization constraints, we mine the constraint that synchronizes transitions in the Petri net with events in the log. To that end, we need to define `log move'.

\begin{definition}[Log Move]\label{def:log_move}
Let $e$ be a log event with data $\{D_1, D_2, \ldots, D_n\} \subseteq \mathcal{D}$. Let $(P,T,F,C)$ be a Petri net, with the same data, and $M$ be a marking of that Petri net. It must hold that $l(e) \in T^T$. A log move for the event $e$ from the current marking of the Petri net, denoted $M\stackrel{e}{\rightarrow}_\ell$, is the firing of transition $l(e)$, with a binding that maps the event's resource to the resource place and the event's case identifier and data to the case place, i.e., $Y = \{r \mapsto r(e), c \mapsto (id(e))\colon d(e)\}$, where $\{r\}=\bullet l(e) \cap P^R$, $\{c\}=\bullet l(e) \cap P^C$, and the `$\colon$' operator is the concatenation of tuples.
\end{definition}

Note that, strictly speaking, this definition assumes that each task transition has exactly one incoming and one outgoing case place. For brevity, we will often violate this constraint in this paper. However, models used in this paper can easily be transformed into models that meet the assumption by strictly separating task transitions from control-flow transitions.

\section{Decision Synchronization Pattern}\label{sec:dec_sync_pattern}

The goal of the approach is to find the constraint that governs the decision synchronization. Informally, a decision synchronization constraint can be observed in the following situations. (1) When a move is made in the log that can also be executed in the given model, the constraint must hold, such that the move can indeed be executed in the given model. (2) When a move is made in the log that cannot be executed in the given model, the constraint must not hold, because then the constraint must block the move that is possible in the model, such that another move (ideally the move that is also performed in the log) becomes possible. For example, in Fig. \ref{fig: Priority process model} at time 5, after starting pre-processing job 2, handling of job 1 does not start according to the log, even though it is possible. This is how we can observe there is a constraint in effect preventing this. More precisely, we assume that the constraint is what makes a difference between being able to perform a particular log move also as a model move. We formally define the constraint that enables this as follows.

\begin{definition}[Constraint-satisfied, replayable Petri net]\label{def:cspn}
Let $L$ be an event log and $(P,T,F,C)$ be a Petri net that is mined from $L$. The constraint-satisfied, replayable Petri net is the Petri net $(P,T,F,C)$, enriched with an additional constraint $G$, in which for all log moves $M\stackrel{e}{\rightarrow}_\ell$, there is a corresponding model move $M\stackrel{Y, t, G}{\rightarrow}$, such that $l(e)=t$, $Y(p^r)=r(e)$ for $\{p^r\}=\{\bullet l(e)\} \cap P^R$, and $Y(p^c)=(id(e), \ldots)$ for $\{p^c\}=\{\bullet l(e)\} \cap P^C$.
\end{definition}


In line with Def.~\ref{def:cspn}, the constraint that makes the Petri net replayable, meets the following criteria.

\begin{proposition}[Decision Synchronization Constraint]\label{prop:dsc}
The additional constraint $G$, which we also call the decision synchronization constraint, that turns the Petri net $N$ into a replayable Petri net, is the constraint that:
\begin{compactenum}
\item is satisfied for bindings $Y$, for which there exists a log move $M\stackrel{e}{\rightarrow}_\ell$, such that there is a model move $M\stackrel{Y, t}{\rightarrow}$ for which $l(e)=t$, $Y(p^r)=r(e)$ for $\{p^r\}=\{\bullet l(e)\} \cap P^R$, and $Y(p^c)=(id(e), \ldots)$ for $\{p^c\}=\{\bullet l(e)\} \cap P^C$; and
\item is not satisfied for bindings $Y$, for which there exists a log move $M\stackrel{e}{\rightarrow}_\ell$, but no model move $M\stackrel{Y, t}{\rightarrow}$ for which $l(e)=t$, $Y(p^r)=r(e)$ for $\{p^r\}=\{\bullet l(e)\} \cap P^R$, and $Y(p^c)=(id(e), \ldots)$ for $\{p^c\}=\{\bullet l(e)\} \cap P^C$.
\end{compactenum}
\end{proposition}

\begin{proof}
The proof directly follows from Def.~\ref{def:cspn}.
\end{proof}

Proposition~\ref{prop:dsc} tells us precisely which data samples need to be collected from the event log to mine the decision synchronization constraint. Consequently, the (approximate) decision synchronization constraint $\hat{G}$ can be mined from a table where each row is a feature/value pair, where the features are taken from the marking (state) of the Petri net and the value is True if and only if the transition should fire. This is captured in Def. \ref{def:dscm}.


\begin{definition}[Decision Synchronization Pattern Mining]\label{def:dscm}
Let $\mathcal{F}$ be a function that extracts features from a binding $Y$. Let process construct $\mathcal{Z}\subseteq F$ specify what places to extract features from. The approximate decision synchronization constraint $\hat{G}$ is the constraint for which $\hat{G}(\mathcal{F}(Y))$ evaluates to True for all bindings $Y$ for which situation 1 of Prop.~\ref{prop:dsc} holds and that evaluates to False for all bindings $Y$ for which situation 2 of Prop.~\ref{prop:dsc} holds. The mined pattern is a combination of process construct $\mathcal{Z}$ and constraint $\hat{G}(\mathcal{F}(Y))$. 
\end{definition}

We allow some flexibility for the feature extraction function $\mathcal{F}$, in that we define it on a combination of a binding $Y$ and a marking $M$. Note that this is not in violation of Petri net syntax and semantics, as a Petri net can be redefined, such that each transition can access - but not change - information about tokens on all places. An example of a feature extraction function, $\mathcal{F}$, extracts: the value of the token in the arrival place, the maximum value of tokens in the queue place, the fraction of these two, and whether the maximum value of tokens in the queue place is enabled, this leads to Tab. \ref{tab:stateoutcome}.

\begin{table}[h!]
\centering
\begin{tabular}{l@{\hskip 0.25cm}l@{\hskip 0.25cm}l@{\hskip 0.25cm}l@{\hskip 0.25cm}l@{\hskip 0.25cm}l}
\hline
   Time &   Arrival Value &   Max Queue Value &   $\frac{Arrival}{Max(Queue)}$ &  \makecell[c]{$Max(Queue)$ \\ \text{ Enabled}} & Constraint   \\
\hline
      5 &             855 &               118 &              7.25 &False &False       \\
     15 &             801 &               855 &              0.17 &True &True        \\
     22 &             222 &               801 &              0.28 &True &True        \\
\hline
\end{tabular}
\captionof{table}{Observation-Outcome Log for Running Example}\label{tab:stateoutcome}
\end{table}

This table is built by replaying the event log from Tab.~\ref{tab: Priority process event log}, recording the states in which `job handling' is enabled according to the model, and creating Tab. \ref{tab:stateoutcome} by applying function $\mathcal{F}$ to the state, and adding `True' when the transition fired, `False' when it did not. Based on this table we can mine the decision synchronization constraint $\hat{G}$. It is easy to see that the constraint $\mathit{arrival\_value} \leq 1.5 \cdot max(\mathit{queue\_value})$ fits this table. Given sufficiently many observations an accurate constraint can be learned.

Clearly, deciding which features to extract from the state using the function $\mathcal{F}$ is not trivial. In the next section, we explain how we do this.
\section{Decision Synchronization Pattern Discovery}\label{sec:Decision Synchronization as Constraints}

Our approach to discovering decision synchronization patterns from an event-log and a Petri net is illustrated in Fig. \ref{fig: Overview decision mechanism mining}. It assumes that an event-log and a process model - in the form of a Petri net - are given. The process model can be created in any way, e.g., by modeling it directly or by discovering it using an existing discovery algorithm, such as the $\alpha$++ miner \cite{Wen2007}. 

The approach discovers decision synchronization constraints $G$ (Prop. \ref{prop:dsc}) on Petri net transitions that constrain which case should happen under which circumstances. The constraints can be defined over many different sets of features $\mathcal{F}$ (Def. \ref{def:dscm}) that are taken directly from case properties, but also from cross-case properties, like the number of waiting cases, the difference in value between cases, or the difference in time duration between cases. 

It is infeasible to create all possible intra- and inter-case features to pass to the constraint mining algorithm. Therefore, we present common decision synchronization patterns $\mathcal{P}$ and, for each pattern $p \in \mathcal{P}$, the set of features $\mathcal{F}_p$ that is relevant for it. Based on the event-log, for a pattern $p \in \mathcal{P}$ and a transition $t \in T$, we can construct a dataset on which to mine a constraint $\hat{G_{pt}}$. For each event $e$ in the original event-log, this dataset contains the feature values $f(e), f \in \mathcal{F}_p$ and whether the constraint holds for that event (Prop. \ref{prop:dsc}). We call this dataset a pattern-transition log. Each $\hat{G_{pt}}$ is mined as a decision tree, from which we extract the decision constraint. Finally, we add the decision constraint to the process model.

\begin{figure}[h]
\centering
\includegraphics[width=\textwidth]{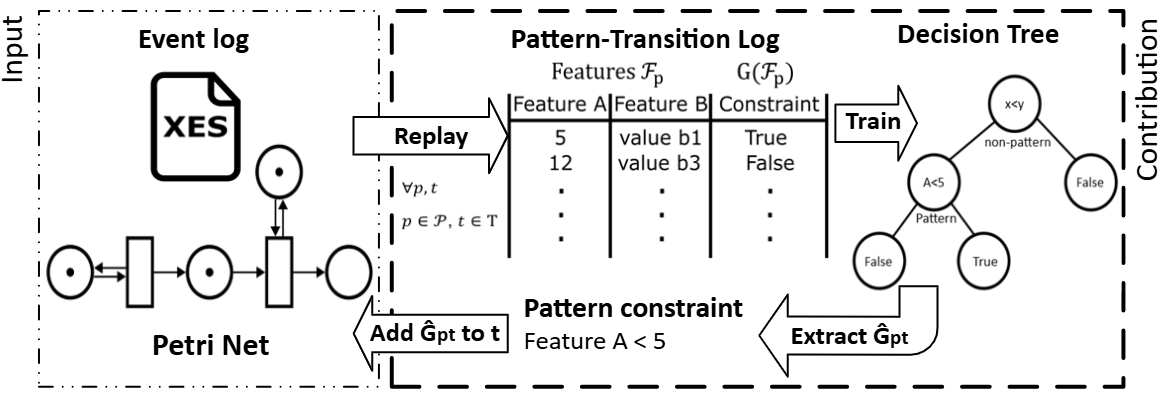}
\caption{Overview Decision Synchronization Pattern Discovery Approach}
\label{fig: Overview decision mechanism mining}
\end{figure}

Accordingly, the remainder of this section is structured as follows. Sect. \ref{subsec:  Pattern Description and Formalization} presents the decision synchronization patterns $\mathcal{P}$ and, for each pattern $p \in \mathcal{P}$, the corresponding features $\mathcal{F}_p$. Sect. \ref{subsec: Enriched event-log} explains how the corresponding pattern-transition log can be derived from an event-log and a process model. Finally, Sect. \ref{subsec: Pattern Discovery} explains how to extract the decision synchronization constraints from the decision tree trained on the pattern-transition log.

\subsection{Pattern Descriptions and Formalization}\label{subsec: Pattern Description and Formalization}

We describe and formalize four decision synchronization patterns inspired by supply chain processes, \textsl{priority, blocking, hold-batch, and choice} in Tab. \ref{tab: Description of decision synchronization patterns}. For each of these patterns, an objective for the process outcome is described as a motivation. The patterns are formalized by describing the features $\mathcal{F}_p$ they use, the related constraints $G$, and the process construct $\mathcal{Z}_p$ in which a constraint can be observed. 

A patterns features $\mathcal{F}_p$ can be described by three categories. First, \textit{AttrVal} represents the colorset $C: P \rightarrow \mathcal{C}$ (Def. \ref{def:binding}) given to tokens in certain places. These could be values, cycle times, error rates, or any other relevant attribute. To evaluate if the token an attribute is linked to is enabled, we use the feature \textit{AttrEnabled}. \textit{NrTokens} represents the number of tokens in a place. \textit{NrTokensEnabled} refers to the subset of tokens that are time-enabled. \textit{TimeUntilNext} represents how long it takes for the next token to be available. 

The features $\mathcal{F}_p$ are used to formalize the patterns as constraints and to represent the state of the process using the pattern-transition log. The constraints specify the relation the features should have with each other and/or some thresholds $a$ and $b$ to be able to fire a guarded transition from the current marking $M\stackrel{Y,t,G}{\rightarrow}$ (Def. \ref{def:binding}). We refer to the guarded transition as $t_g$. 

The process construct $\mathcal{Z}_p$ (Def. \ref{def:dscm}) describes how the flows $F \subseteq (P\times T) \cup(T\times P)$ (Def. \ref{def:cpn}) between places and transitions of interest should be connected to constitute a pattern. The description and formalization are provided in Tab. \ref{tab: Description of decision synchronization patterns}. Each pattern is illustrated by an example in Fig. \ref{fig: decision patterns}. 

\begin{table}
\centering
\renewcommand\cellalign{lc}
\renewcommand\cellgape{\Gape[4pt]}
\scriptsize
\begin{tabular}{|>{\centering\arraybackslash}m{1.5cm}|
                 >{\centering\arraybackslash}m{2.9cm}|
                 >{\centering\arraybackslash}m{2cm}|
                 >{\centering\arraybackslash}m{2.9cm}|
                 >{\centering\arraybackslash}m{2.5cm}|}
\hline
\textbf{Pattern} & \textbf{Priority} & \textbf{Blocking} & \textbf{Hold-Batch} & \textbf{Choice} \\ \hline

\makecell[c]{\textbf{Objective}} 
& Prioritization of jobs based on attributes like a high value, faster processing, or a low failure rate 
& Downstream locations in a process are full or incur high unit holding cost 
& Reduce unit processing costs of a batch by waiting short time period for more cases to arrive 
& Avoid depletion of a common resource when two activities require this same resource 
\\ \hline


\textbf{Features $\mathcal{F}_p$} 
& AttrVal, AttrEnabled
& NrTokens 
& NrTokensEnabled, TimeUntilNext
& TimeUntilNext \\ \hline

\makecell[c]{\textbf{Constraint} \\
             $\forall\, i, j \in P,$\\
             $\ i \ne j$} 
& \makecell[c]{
    $k_x(\text{AttrVal}_i)\leq a \times$ \\
    $k_y(\text{AttrVal}_j) \land$\\ \\
    $k_x(\text{AttrEnabled}_i)$ \\
    $==True$\\ \\
    $k_x, k_y \in $\\
    $\{argmin, argmax\}$}
& \makecell[c]{
    $\text{NrTokens}_i\leq a$}
& \makecell[c]{
    $\text{NrTokensEnabled}_i \geq a$\\ 
    $\land$\\ 
    $\text{TimeUntilNext}_i \geq b$}
& \makecell[c]{
    $\text{TimeUntilNext}_i \geq b$} \\ \hline

\makecell[c]{\textbf{Process}\\ \textbf{construct} \\ $\mathcal{Z}_p$ \\
             $\exists t_g,t' \in T,\,$ \\ $\exists p_i, p_j \in P$} 
& \makecell[c]{
    $(p_i, t_g) \in F \land$\\
    $(p_j, t') \in F \land $\\ 
    $(t', p_i) \in F \land$\\
    $p_i \neq p_j, t_g \neq t' $}
& \makecell[c]{
    $(t_g, p_i) \in F$}
& \makecell[c]{
    $(p_i, t_g) \in F$}
& \makecell[c]{
    $(p_i, t') \in F \land$ \\ 
    $(p_j, t') \in F \land$\\
    $(p_j, t_g) \in F \land$ \\
     $p_i \neq p_j, t_g \neq t'$
    } \\ \hline
\end{tabular}
\caption{Description and Formalization of Decision Synchronization Patterns}
\label{tab: Description of decision synchronization patterns}
\end{table}

\begin{figure}
  \centering
  \begin{subfigure}{0.48\textwidth}
    \centering
    \includegraphics[width=\linewidth]{Figures/Priority.png}
    \caption{Priority}\label{subfig: priority}
  \end{subfigure}
  \begin{subfigure}{0.48\textwidth}
    \centering
    \includegraphics[width=\linewidth]{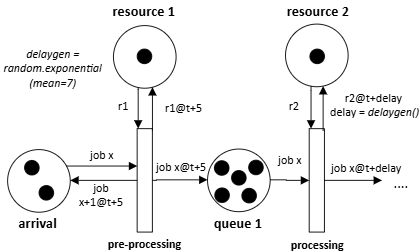}
    \caption{Blocking}\label{subfig: blocking}
  \end{subfigure}

  \begin{subfigure}{0.48\textwidth}
    \centering
    \includegraphics[width=\linewidth]{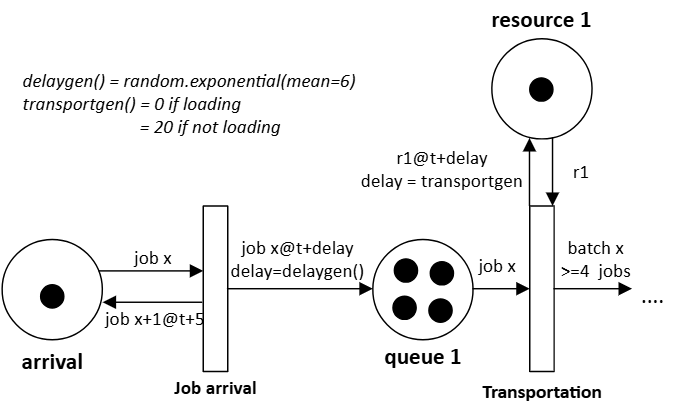}
    \caption{Hold-batch}\label{subfig: hold-batch}
  \end{subfigure}
  \begin{subfigure}{0.48\textwidth}
    \centering
    \includegraphics[width=\linewidth]{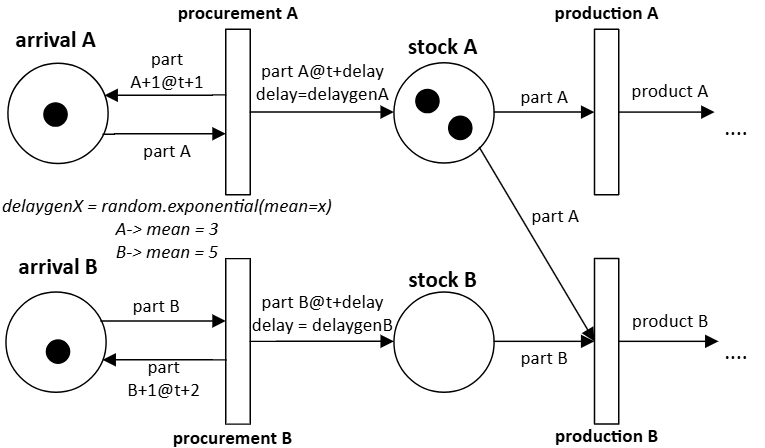}
    \caption{Choice}\label{subfig: choice}
  \end{subfigure}
  \caption{Process Models per Decision Synchronization Pattern}\label{fig: decision patterns}
\end{figure}

\mypar{Priority} In Fig. \ref{subfig: priority}, if a job arrives with a value more than 1.5 times higher than the most valuable job in queue 1, one needs to wait for this job to move to queue 1 and be processed in the job handling transition. In this case, the $\mathcal{F}_p$ corresponds to \textsl{value} and \textsl{max\_value\_enabled}, and the constraint on transition $t_g =$ pre-processing can be formalized as $max(arrival\_value) \leq 1.5 \times max(queue1\_value) \land argmax(queue1\_value) == True$. The process construct is in line with the definition for priority, since queue 1 is an input of guarded transition $t_g$, and arrival is an input to transition $t'$ such that it has an outflow to queue 1.


\mypar{Blocking} In Fig. \ref{subfig: blocking}, when queue 1 reached its capacity of five tokens, pre-processing is blocked. $\mathcal{F}_p$ is \textsl{number of tokens}, which can be formalized as a constraint on transition $t_g =$ pre-processing of $queue\_1\_nr\_tokens \leq 5$. The process construct specifies that $t_g$ should be an outflow of the place, which is the case for queue 1.


\mypar{Hold-batch} In Fig. \ref{subfig: hold-batch}, jobs are usually transported from a batch size of 4 but if a fifth job is ready in less than 2 time units, we wait for the fifth one to be included in the current batch. $\mathcal{F}_p$ for this example is \textsl{number of tokens enabled} and \textsl{time until next}. This can be formalized as a constraint on transition $t_g=$ transportation of $queue\_1\_nr\_tokens\_enabled \geq 4 \land  queue\_1\_time\_until\_next < 2$. Place $queue\_1$ is an inflow to guarded transition $t_g$, so the process construct is in line with the formalization.


\mypar{Choice} In Fig. \ref{subfig: choice}, when product B cannot be produced because there is no stock B, but it takes less than three time units for tokens to become available in stock B, hold cases of stock A for production A to allow the choice between product A and product B. $\mathcal{F}_p$ is \textsl{time until next}. This can be formalized with guarded transition $t_g$ = production A with constraint $stock\_B\_time\_until\_next \geq 3$. The process construct follows the formalization, since production A shares an input place with production B, and the place under consideration in the constraint is an inflow for production B that has no inflows to production A.

\subsection{Pattern-transition Log}\label{subsec: Enriched event-log}

To find the decision synchronization constraints, we will train a decision tree using a pattern-transition log. The first step is to generate this log from an event-log $L$ and a Petri net $P$. Note that the behavior in $L$ contains the constraints $G$ that we want to discover, and constraints $G$ are not reflected in $P$ (i.e. $P$ is not constraint-satisfied according to Def. \ref{def:cspn}). To discover if a constraint applies on some transition $t \in T$ according to some pattern $p \in \mathcal{P}$, we will create a pattern-transition log for $p$ and $t$. This will be done by replaying $L$ over $P$ and collecting the features $\mathcal{F}_p$. To know what pattern-transition logs should be generated, the incoming and outgoing flows $F$ to transition $t$ should be in line with the process construct for a specific pattern $p$, as outlined in Tab. \ref{tab: Description of decision synchronization patterns}.

Proposition~\ref{prop:dsc} and Def.~\ref{def:dscm} contain the primary rules for extracting a pattern-transition log from an event-log for pattern $p$ and transition $t$. Starting from an event-log $L$ and a Petri net $P$ in initial marking $M$, we iterate over all events $e\in L$. For each event $e$, we find the enabled binding $M\stackrel{Y,t}{\rightarrow}$, $l(e)=t$, with the highest similarity $\mathit{sim}(Y, e)$ to the event. This binding is added to the pattern-transition log with features $\mathcal{F}_p(Y,M)$ and the constraint value $\hat{G}$ set to True. The other bindings are added with features $\mathcal{F}_p(Y,M)$ and the constraint value $\hat{G}$ set to False. The similarity function can be defined in different ways. For example, in Prop.~\ref{prop:dsc} we require an exact match between the values established in the binding and the values associated with the event. In our evaluation, we require an exact match for the case id variable and a time as close as possible to the event time, after the event time, and the other variables as close as possible to the event variables. An example of a fragment of the pattern-transition log for the ``handling'' transition in the priority process was given in Tab. \ref{tab:stateoutcome}. This type of log is used in the next step to train a decision tree and discover approximate constraint $\hat{G}$ (Def. \ref{def:dscm}). 

\subsection{Pattern Constraint Discovery}\label{subsec: Pattern Discovery}
Based on the pattern-transition log a decision tree is trained to classify whether the constraint is aligned (Prop. \ref{prop:dsc}) for the transition in the pattern-transition log or not. Fig. \ref{fig:Decision tree priority} shows an example of the decision tree trained for the handling transition in the priority process. The steps for discovering the decision synchronization patterns from such a tree are outlined beside Fig. \ref{fig:Decision tree priority}. The splitting criteria in the path from the root node to the leaf nodes where the constraint is not aligned contain the constraints of decision-synchronization patterns. The path from the leaf nodes to the root node is traced when the constraint is classified as False, the number of samples is greater than threshold $\tau_s$, and the Gini impurity value is lower than threshold $\tau_g$. Thresholds $\tau_s$ and $\tau_g$ ensure that only high-confidence predictions are used to discover pattern constraints. The splitting criteria in the nodes on the path are evaluated against constraints in the pattern formalization. This step ensures that non-pattern splitting criteria, which could represent either noise or other process-constraining effects, are not considered as a pattern constraint. Only node splits that cause a decrease in the Gini value $\Delta Gini$ between the current and the next node are considered. If a node is present in two paths, only the splitting criteria with the highest absolute $\Delta Gini$ is retained. In case of multiple constraints representing a single pattern (priority and hold-batch), the pattern is only returned if both have been discovered from the splitting criteria. The returned pattern constraint is added to the original Petri net to make it a constraint-satisfied, replayable Petri net.

\begin{figure}[ht]
\begin{minipage}[t]{0.48\textwidth}
\scriptsize
\vspace{-35 mm}
\textbf{Given: Decision Tree $T$; thresholds $\tau_s,\tau_g$; pattern constraints $\mathcal{P}$} \\
\begin{enumerate}
    \item Select leaf nodes with Class=False, \\
          $\text{samples} \ge \tau_s$, $\text{Gini} \le \tau_g$.
    \item For each leaf:
    \begin{enumerate}
        \item Trace path to root.
        \item Keep nodes where comparison operator matches $\mathcal{P}$.
        \item Retain nodes with $\Delta Gini$<0.
        \item Retain duplicates with $\max(|\Delta Gini|)$.
        \item Verify constraints in $\mathcal{P}$ are covered.
    \end{enumerate}
\end{enumerate}
\textbf{Return: Pattern = Remaining splitting criteria}

\end{minipage}
\hfill
\begin{minipage}[t]{0.48\textwidth}
\centering
\includegraphics[width=\linewidth]{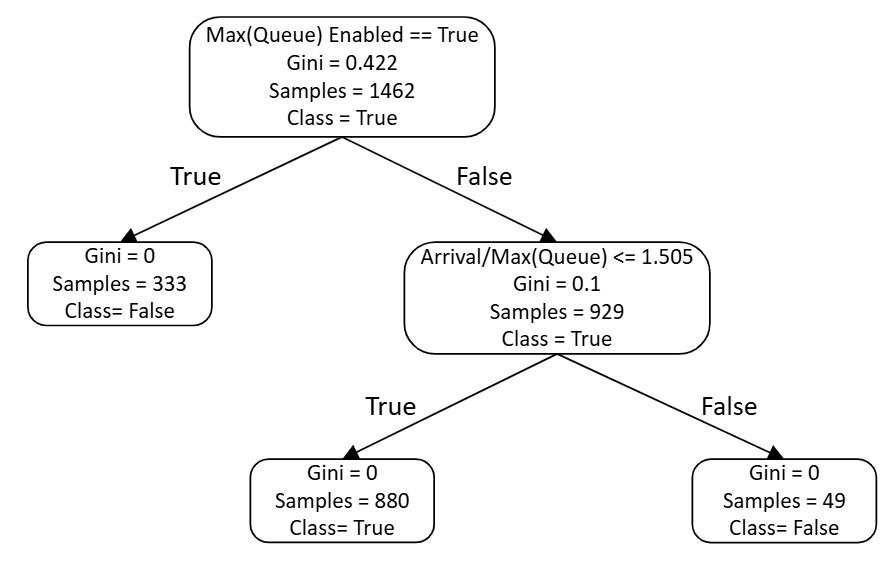}
\caption{Decision Tree Priority Pattern}
\label{fig:Decision tree priority}
\end{minipage}
\end{figure}

To illustrate the procedure, lets follow the steps in the decision tree from Fig. \ref{fig:Decision tree priority}. The procedure was followed using a $\tau_s$ of 10 and a $\tau_g$ of 0.1. The leaf nodes with Class=False are found on the True branch after the first split, and the False branch after the second split. The minimum number of samples are, respectively, 333 and 49, which are greater than $\tau_s$. Both Gini values are 0, being lower than $\tau_g$. Next, we trace the path from these leaf nodes to the root node and collect the splitting criteria. We find the following splits: 1) $Max(Queue) Enabled == True$, 2) $Max(Queue) Enabled == False$, 3) $Arrival/Max(Queue) > 1.505$. In this tree, all node comparison operators match $\mathcal{P}$, and all nodes have a negative $\Delta Gini$. We do have duplicate constraints 1) and 2), so we retain the constraint with the highest $|\Delta Gini|$. For 1) $|\Delta Gini| = 0.422$, and 2) $|\Delta Gini| = 0.322$, so we retain only 1). For the priority pattern, two parts of the constraint must be there to constitute a decision synchronization pattern. Since both $Max(Queue Enabled) == True$ (token enabled condition) and $Arrival/Max(Queue) > 1.505$ (attribute condition), are there, the decision synchronization pattern is returned.

\section{Evaluation}\label{sec:evaluation}
The performance of the pattern discovery procedure described in Sect. \ref{sec:Decision Synchronization as Constraints} is evaluated along two lines. First, using the single-pattern processes in Fig. \ref{fig: decision patterns} and comparing the constraints returned against the constraints implemented in the process models. Second, we use a more complex fictional supply chain process in which multiple patterns are imposed. The implementation and evaluation results, as well as both datasets including all models, parameterizations, and constraints, are provided as supplementary material\footnote{\url{https://github.com/TijmenKuijpers/Decision_Sync_Patterns.git}\label{footnote: git page}}.

\subsection{Single-Pattern Evaluation}\label{subsec: Single Pattern Evaluation}
We use the process models as given in Fig. \ref{fig: decision patterns} to evaluate the approach for a single-pattern setting, i.e., a setting in which per process only one decision synchronization pattern is present. In Tab. \ref{tab: Single-pattern modeled and discovered constraints}, the constraints that are imposed on the process model, and the constraints that are discovered using the pattern discovery algorithm are shown for each pattern. All constraints discovered closely match or are identical to the original constraint. No additional constraints that are not present in the original model are discovered. The constraints that use the \textit{NrTokens} feature base have threshold values that differ by 0.5 from the modeled constraints. These values are equivalent to the values in the modeled constraints, because the \textit{NrTokens} is a discrete attribute and the set of accepted values is the same. This evaluation shows that when a decision synchronization pattern is present for a single transition in a small process model, the pattern discovery procedure can correctly identify the constraints that belong to the patterns. In Sect. \ref{subsec:multi-pattern-eval} we test the robustness of the approach with larger process models containing multiple decision synchronization patterns at once. 

\begin{table}[ht]
\centering
\scriptsize
\begin{tabular}{|l|l|l|l|}

\hline
\makecell[c]{\textbf{Pattern}} 
& \makecell[c]{\textbf{Event}} 
& \makecell[c]{\textbf{Constraint Model}} 
&  \makecell[c]{\textbf{Constraint Discovered}} \\ \hline
Priority         
& Job handling   
& \makecell[c]{$AttrVal_{arrival} \leq 1.5 \times max(AttrVal_{q1})$\\
                $\land$\\
                $argmax(ValEnabled_{q1}) = True$}                    
& \makecell[c]{$\frac{AttrVal_{arrival}}{\max(Val_{q1})}\leq 1.497$\\
                $\land$ \\
                $argmax(ValEnabled_{q1}) = True$}                 

\\ \hline
Blocking         
& Pre-processing 
&  \makecell[c]{$NrTokens_{q1} < 5$}                   
& \makecell[c]{$NrTokens_{q1} \leq 4.5$}                 
                                                           
\\ \hline
\makecell[c]{Hold-\\batch}      
& Transportation 
&  \makecell[c]{$NrTokensEnabled_{q1} > 3$\\
                $\land$ \\
                $TimeUntilNext_{q1} > 2 $}                   
& \makecell[c]{$NrTokensEnabled_{q1} > 3.5$ \\
                $\land $\\
                $TimeUntilNext_{q1} > 2$}                  
                                                           
\\ \hline
Choice           
& Production A   
& \makecell[c]{$TimeUntilNext_{arrivalB} > 2$}                   
& \makecell[c]{$TimeUntilNext_{arrivalB} > 1.915$}               
                                                        
\\ \hline
\end{tabular}
\caption{Single-Pattern Evaluation: Modeled and Discovered Constraints}\label{tab: Single-pattern modeled and discovered constraints}
\end{table}

\newpage
\subsection{Multi-Pattern Evaluation - a Fictional Supply Chain}\label{subsec:multi-pattern-eval}

\mypar{Setup}
The multi-pattern evaluation follows a fictional supply chain process as illustrated in Fig. \ref{fig: Fictional Supply Chain Process Model}. The process involves the ordering, production, and transportation of phones and game computers from phone cases (pc), chips (c), and game cases (gc). The final products are distributed to two phone warehouses and one game warehouse serving customers in the Netherlands (NL) and Germany (DE). Each pattern is imposed on the process exactly once. The priority pattern is imposed on phone production to allow any tokens with a higher value in ``ordered phone cases'' to be produced first. Priority values are either 1 for high priority, or 0 for low priority. Ordering of game cases is blocked when there are already three tokens in ``stock game cases''. The transportation occurs in batches of five, and holds for additional tokens when they are enabled for ``phone NL'' in less than one time unit. Finally, the choice pattern is imposed on game production so that the transition is not fired when a phone case in ``stock phone cases'' is enabled in less than half a time unit. Some delays in the process are stochastic to allow for various process executions. To account for this stochasticity in the evaluation, the experiment is reproduced 10 times. The detailed parameterization and all experimental results are added to the supplementary materials. 

\begin{figure}
\centering
\includegraphics[width=\textwidth]{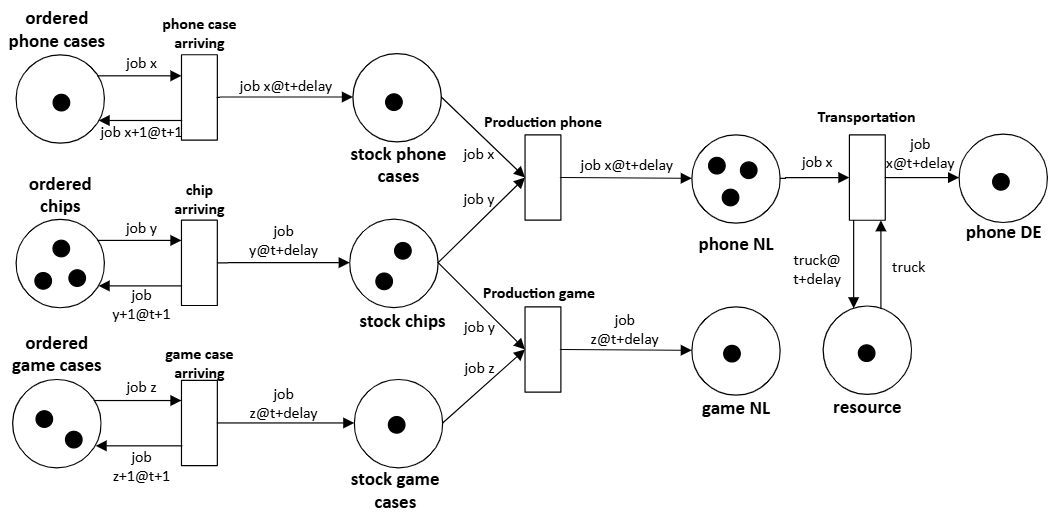}
\caption{Fictional Supply Chain Process Model}
\label{fig: Fictional Supply Chain Process Model}
\end{figure}

\mypar{Results}
For one of these experiments, the modeled and discovered constraints are stated in \autoref{tab: Multi-pattern evaluation: Modeled and Discovered Constraints}. All constraints are discovered in all experiments. No additional constraints are discovered that are not in the original model. The accuracy for the threshold value of the constraints varies depending on the number of samples there are for the pattern execution. Due to the parameterization of the process the number of samples, mainly for the priority and hold-batch pattern, was sometimes low. This results in less accurate results in some experiments for these patterns. Note that the discovered priority constraint is $\frac{max(AttrVal_{ordered\_pc})}{\max(AttrVal_{stock\_pc})}\leq 50.5$. This high value results from substituting 0 for 0.01 to avoid division by zero, so the possible values of this fraction are 0, 1, or 100. Thus, the constraint can be read as $\frac{max(AttrVal_{ordered\_pc})}{\max(AttrVal_{stock\_pc})}= 0$ or $1$, which is consistent with the modeled constraint.

\begin{table}[ht]
\centering
\scriptsize
\begin{tabular}{|l|l|l|l|}

\hline
\makecell[c]{\textbf{Pattern}} 
& \makecell[c]{\textbf{Event}} 
& \makecell[c]{\textbf{Constraint Model} }
&  \makecell[c]{\textbf{Constraint Discovered}} \\ \hline
Priority         
& \makecell[c]{Production \\ phone}   
& \makecell[c]{$AttrVal_{ordered\_pc} \leq $\\
                $max(AttrVal_{stock\_pc})$\\
                $\land$\\
                $argmax(ValEnabled_{stock\_pc})$ \\ 
                $= True$}                    
& \makecell[c]{$\frac{max(AttrVal_{ordered\_pc})}{\max(AttrVal_{stock\_pc})}\leq 50.5$\\
                $\land$ \\
                $argmax(ValEnabled_{stock\_pc}) $\\
                $= True$}                        
                                                            
\\ \hline
Blocking         
& \makecell[c]{Game case \\ arriving} 
&  \makecell[c]{$NrTokens_{stock\_gc} < 3$}                   
& \makecell[c]{$NrTokens_{stock\_gc} \leq 2.5$}                 
                                                           
\\ \hline
\makecell[c]{Hold-\\batch}      
& \makecell[c]{Transpor- \\ tation} 
&  \makecell[c]{$NrTokensEnabled_{phone\_NL} > 2$\\
                $\land$ \\
                $TimeUntilNext_{phone\_NL} > 1 $}                   
& \makecell[c]{$NrTokensEnabled_{phone\_NL} > 2.5$ \\
                $\land $\\
                $TimeUntilNext_{phone\_NL} > 1.08$}       
                                                           
\\ \hline
\makecell[c]{Choice}        
& \makecell[c]{Production\\ game}   
& \makecell[c]{$TimeUntilNext_{stock\_pc} > 0.5$}                   
& \makecell[c]{$TimeUntilNext_{stock\_pc} > 0.5$}

\\ \hline
\end{tabular}
\caption{Multi-pattern evaluation: Modeled and Discovered Constraints}\label{tab: Multi-pattern evaluation: Modeled and Discovered Constraints}
\end{table}

\mypar{Discussion}
The multi-pattern evaluation shows that the decision synchronization pattern discovery procedure works for complex processes, i.e. with multiple places, transitions, and patterns, even with stochasticity added in the transition delays. We note that the accuracy of the results largely depends on the number of samples per pattern execution in the original log. There is a tradeoff between the number of accepted leaf samples ($\tau_s$) in the decision tree and the accuracy and completeness of the returned patterns.  Further generalization of the procedure can be evaluated in various experiments. Possible experiments include overlapping patterns (same transition or feature space), stochastic pattern execution, guards not related to decision synchronization patterns. A mathematical proof and implementation of methods in a case study are part of future work to show the methods generalizability.

\section{Related Work}\label{sec:rel_work}
In this work, we extract constraints mirroring the reasons for decision synchronization from event logs using decision trees. Decision mining approaches like \cite{DBLP:conf/bis/BazhenovaBW16,DBLP:conf/caise/MannhardtLRA16,DBLP:conf/bpm/RozinatA06} are similar in terms of technique and aim, yet our work differs since we discover decision rules that can span across multiple process instances. In \cite{Winter2020a} constraints spanning multiple instances are defined and in \cite{Winter2020b} discovered from event logs. The main difference of \cite{Winter2020a,Winter2020b} to our work is that we consider different types of synchronization mechanisms and derive explicit decision rules underlying those. 

Approaches that account for capturing collaboration aspects as we are considering include, i.a., collaboration mining as presented in \cite{Benzin2025}, which focuses on message exchange and resource sharing using Petri nets. A different variation on collaboration mining, collaborative business process discovery in BPMN 2.0 \cite{Pea2024}, focuses on message exchange and choreography. Another related research direction is object-centric process mining \cite{vanderAalst2019,DBLP:journals/fuin/AalstB20}, which was demonstrated to be suitable to uncover overstocking patterns in a supply chain case study \cite{Kretzschmann2024}. Compared to the latter, our focus and consequently the types of patterns considered goes beyond only considering overstock.

Another line of research related to our work is declarative process mining, e.g., \cite{DBLP:conf/cidm/BernardiCM14,DBLP:conf/icpm/ChristfortRFHS24,DBLP:conf/cidm/CiccioM13,DBLP:conf/cidm/MaggiMA11}, which discover the rules and constraints that determine execution of processes in general. Our work differs in that it discovers additional constraints in imperative models, also considering inter-case constraints. In doing so, the techniques that are used are fundamentally different.

\section{Conclusion}\label{sec:conclusion}

Decision synchronization is crucial for achieving desired process outcomes, yet there are only a few approaches that can discover synchronization patterns from event logs, in particular those relying on properties of multiple cases. In this paper, we contribute to addressing this gap by presenting an approach, based on Petri nets and decision trees, capable of discovering different decision synchronization patterns inspired by supply chain processes. A pattern consists of certain process constructs in combination with constraints determining which case to execute when.
The evaluation of our approach shows that we can discover explicit constraints that describe these decision synchronization patterns in small to medium-sized processes. We evaluated four single-pattern and a multi-pattern process to show generalizability when multiple patterns are expressed in a single process. To further test the generalizability and robustness of this method, additional experiments are subject to future work. For further generalizability, experiments include trying different patterns, adding noise to the logs in the form of additional non-pattern constraints, stochastic conformance to patterns, and examining possible interference when multiple patterns overlap (i.e., apply to the same places, features, or transitions). Finally, the methods should be applied to a real-life event log to test the robustness.

%
%
%
\bibliographystyle{splncs04}
\bibliography{bibliography}

\end{document}